\documentclass[twocolumn,envcountsame]{svjour3}
\smartqed
\usepackage{amsmath,amssymb}
\usepackage{graphicx}
\usepackage{subfig}
\usepackage{enumerate}
\usepackage{times}
\usepackage{hyperref}
\usepackage{algorithm,algpseudocode}
\usepackage[all]{xy}

\usepackage{xcolor}
\usepackage[textsize=footnotesize,colorinlistoftodos]{todonotes}

\usepackage{tikz}
\usepackage{tkz-berge}
\usetikzlibrary{shapes,arrows}
\usetikzlibrary{petri,topaths}
\usetikzlibrary{calc}

\tikzstyle{decision} = [diamond, draw, fill=blue!20, 
    text width=4.5em, text badly centered, node distance=3cm, inner sep=0pt]
\tikzstyle{block} = [rectangle, draw, fill=blue!20, rounded corners, 
                     minimum width=10em, minimum height=4em,
                     text centered, text width=6em,
                     node distance=4.5cm]
\tikzstyle{line} = [draw, -latex', line width=3pt]
\tikzstyle{cloud} = [draw, ellipse,fill=red!20,
                     minimum width=4em, minimum height=4em, 
                     node distance=3cm,]

\newcommand{\RR}{\mathbf{R}}
\newcommand{\norm}[1]{\|#1\|}

\newcommand{\R}{\mathbb{R}}

\DeclareMathOperator{\E}{E}
\DeclareMathOperator{\Var}{Var}

\DeclareMathOperator{\argmin}{argmin}
\DeclareMathOperator{\KL}{KL}
\DeclareMathOperator{\const}{const}
\newcommand{\entropy}{\mathbb{H}}

\title{
An extended Perona-Malik model based on probabilistic models\thanks{This material was based upon work partially supported by the National Science Foundation under Grant DMS-1127914 to the Statistical and Applied Mathematical Sciences Institute. 
    Any opinions, findings, and conclusions or recommendations expressed in this material are those of the author(s) and do not necessarily reflect the views of the National Science Foundation.}}
\author{L. M. Mescheder \and D. A. Lorenz}
\institute{L. M. Mescheder\at Autonomous Vision Group, MPI T\"ubingen, 72076 T\"ubingen, Germany
\email{lmescheder@tuebingen.mpg.de}
\and D. A. Lorenz\at Institute for
  Analysis and Algebra, TU Braunschweig, 38092 Braunschweig, Germany,
  \email{d.lorenz@tu-braunschweig.de}}

\begin{document}
\maketitle

\begin{abstract}
The Perona-Malik model has been very successful at restoring images from noisy input.
In this paper, we reinterpret the Perona-Malik model in the language of Gaussian scale mixtures and derive some extensions of the model.
Specifically, we show that the expectation-maximization
(EM) algorithm applied to Gaussian scale mixtures leads
to the lagged-diffusivity algorithm for computing stationary points of the
Perona-Malik diffusion equations.
Moreover, we show how mean field approximations to these Gaussian scale mixtures lead
to a modification of the lagged-diffusivity algorithm that better captures the
uncertainties in the restoration.
Since this modification can be hard to compute in practice
we propose relaxations to the mean field objective to make the algorithm
computationally feasible.
Our numerical experiments show that this modified lagged-diffusivity algorithm often
performs better at restoring textured areas and fuzzy edges than the unmodified
algorithm. As a second application of the Gaussian scale mixture framework, we
show how an efficient sampling procedure can be obtained for the probabilistic model,
making the computation of the conditional mean and other expectations algorithmically
feasible.
Again, the resulting algorithm has a strong resemblance to the
lagged-diffusivity algorithm.
Finally, we show that a probabilistic version of the Mumford-Shah segementation model can be obtained in the same framework with a discrete edge-prior.


\end{abstract}

\section{Introduction}
\label{sec:intro}
In mathematical image processing, one is often given some (linear) forward operator $A: X \to Y$ and some noisy observation $v_n$. Our goal is to reconstruct a noise-free image that explains the observed data well, i.e. an approximate solution $u$ to $ A u \approx v_n$.
For example, $A$ can denote a convolution operator and $v_n$ a noisy measurement of the resulting blurry image.

Several approaches exist to solve problems of this type.
One popular approach is via variational regularization~\cite{scherzer2009variational,bertero1998inverseproblems}
where one formulates a minimization problem and seeks the solution as a minimizer of a weighted sum of data fit (or discrepancy) term and a regularization term.
The weighting factor is called regularization parameter and controls the trade-off between data fit and regularization.
A related approach is to reconstruct $u$ via the solution of an appropriate partial differential equation.
One of the earliest such partial differential equation is used in the Perona-Malik diffusion algorithm~\cite{perona1990scalespace,nordstrom1990biased,weickert1998anisotropic} where one solves the image reconstruction problem by computing stationary points of the nonlinear diffusion problem
\begin{equation}\label{eq:pde-formulation}
 \partial_t u +   \tfrac1{\sigma^{2}}A'(Au - v_n)=\nabla \cdot \left( f\left(\frac{1}{2}|\nabla u|^2\right) \nabla u \right)
\end{equation}
where $\sigma$ is an estimate for the noise level and $f:\R\to\R$ is some nonlinear, positive and  decreasing function that stops the diffusion at points where $|\nabla u(x)|$ is large.
In fact, the two approaches are closely related since the stationary points of this equation are indeed
stationary points of the optimization problem
\begin{equation}\label{eq:var-formulation}
 \min_u \frac{1}{2\sigma^{2}}\|Au - v_n\|^2 + \int F\left(\frac{1}{2}|\nabla u(x)|^2\right) \mathrm d x,
\end{equation}
where $F$ is any function with $F^\prime = f$.

An alternative approach is given by Bayesian statistics, where  we define a prior distribution $p(u)$ on the space of possible images and model the forward model including the noise by the likelihood term $p(v_n \mid u)$. The posterior distribution is then given by
\begin{equation}\label{eq:bayes}
 p(u \mid v_n) = \frac{p(u) p(v_n \mid u)}{\int p(u) p(v_n \mid u) \mathrm d u}.
\end{equation}
In theory, the posterior distribution carries all information that we put into the model
but as a distribution on the space of all possible images it is very high dimensional 
and it is not straightforward to extract useful information.
One piece of information is the so-called maximum-a-posterior (MAP) estimate which is simply the image $u$ that maximizes the posterior, i.e. the solution of
\begin{equation}
  \max_{u}p(u) p(v_n \mid u)
\end{equation}
which, by taking the negative logarithm, again amounts to a minimization problems.
Beyond the MAP estimate there are other quantities that can be computed from the posterior distribution, see e.g.~\cite{kaipio2005statistical} for further information.

In this work we derive a probabilistic model for the image reconstruction problem and analyze different methods to extract information for the reconstruction.
We show that using the expectation-maximization algorithm \cite{dempster1977maximum} reproduces the lagged-diffusivity approach for the Perona-Malik equation~\cite{vogel1996iterative,chan1999convergence}. However, other methods exists and in this work we specifically analyze a mean-field approach and a sampling strategy which lead to different methods that extend the Perona-Malik model.

Some advantages of the probabilistic approach to the image reconstruction problem are that uncertainty in the reconstruction is explicitly part of the model which can lead to more meaningful reconstructions.
Moreover, estimates for the uncertainty can be computed, e.g. marginal variances.
Also, the probabilistic model comes from a sound foundation that combines well with learning approaches in which parts of the reconstruction model itself are learned from the data as in~\cite{chen2015learning}.
As a result, the probabilistic approach opens the full toolbox of Bayesian statistics and allows to employ further techniques, e.g. from the field of model selection~\cite{yu2010imagemodelselection}.

Probabilistic models have been used in image processing since long ago and we only give a few pointers.
Geman and Geman \cite{geman1984stochastic} were one of the first to propose a probabilistic
framework for image restoration tasks. They use simulated annealing for optimization, which
is often not very efficient.
More recently, and closer to our work, Schmidt et al. \cite{schmidt2010agenerative} use Gaussian scale mixtures to train generative image models for denoising and inpainting tasks.
A classical resource for the statistical treatment of inverse problems is given in \cite{kaipio2005statistical}.
Other probabilistic approaches to the Perona-Malik model differ from our approach: On the one hand~\cite{krim1999nonlinear,bao2004smart} build on stochastic partial differential equations and random walks on lattices and not on Bayesian statistics.
On the other hand~\cite{pizurica2006bayesian} takes a Bayesian approach to non-linear diffusion but derives a new diffusivity function $f$ by these means.

The article is organized as follows.
In Section~\ref{sec:gsm} we introduce our model based on Gaussian scale mixtures, introduce the notion of an exponential pair of random variables and derive a few properties of our model.
Section~\ref{sec:methods} gives three methods to infer information from the posterior distribution, namely expectation-maximization, mean field methods and sampling.
In Section~\ref{sec:applications} we present results of our methods and in Section~\ref{sec:conclusion} we draw some conclusions.



\section{Gaussian Scale Mixtures}
\label{sec:gsm}
In the following we denote by $\Omega$ the domain of the image which we assume to be a regular rectangular grid,
by $x\in\Omega$ the pixels and by $u$ the image, i.e $u(x)\in \RR$ is the gray value of $u$ at pixel $x$.
We denote by $\nabla u(x)\in\RR^{2}$ the discrete gradient of $u$ at $x$, i.e.
\[
 \nabla u(x_1, x_2) := \begin{pmatrix}                   
  u(x_1 + 1, x_2) - u(x_1, x_2) \\
  u(x_1, x_2+1) - u(x_1, x_2)
\end{pmatrix}.
\]
For a vector field $v:\Omega\to\RR^{2}$ we denote by $\nabla\cdot v(x)$ the discrete divergence of $v$ at $x$, i.e.
\begin{multline*}
 \nabla \cdot v(x_1, x_2) := v_1(x_1, x_2) - v_1(x_1-1,x_2) \\
  + v_2(x_1, x_2) - v_2(x_1, x_2-1).
\end{multline*}
Moreover, boundary conditions are chosen such that $\nabla$ and $-\nabla\cdot$ are adjoint to each other.

A Gaussian model for the magnitude of the gradient of $u$ corresponds to a density
\footnote{Strictly speaking, $p(u)$ is not a proper probability density, as it is not normalizable. In Bayesian statistics, such probability densities are often referred to as \emph{improper probability densities}. In practice, only the posterior density $p(u \mid v_n)$ is used for calculations, which generally defines a normalizable probability density.}
\[
p(u) \propto \exp\Big(-\tfrac\lambda2\sum_{x\in\Omega}|\nabla u(x)|^{2}\Big).
\]
If we further assume that our observation $v_{n}$ is obtained by $Au$ plus Gaussian white noise with variance $\sigma^{2}$, the corresponding likelihood is also Gaussian
\[
p(v_{n}\mid u) \propto \exp\Big(-\tfrac1{2\sigma^{2}}\|Au-v_{n}\|^{2}\Big).
\]
Consequently, the posterior density is given by
\[
p(u\mid v_{n}) \propto \exp\big(-\tfrac1{2\sigma^{2}}\|Au-v_{n}\|^{2}-\tfrac\lambda2\sum_{x\in\Omega}|\nabla u(x)|^{2}\big).
\]
It is well known, that this model is not well-suited to denoise or deblur natural images, e.g. one observes that the MAP estimate for $u$ cannot have sharp edges anymore. 

We propose a \emph{Gaussian scale mixture} (GSM) for the magnitude of the gradient, that is a Gaussian model with additional latent scale variables $z(x)$, i.e. prior models that have a density of the form
\begin{equation}
p(u,z)\propto\exp\left(-\sum_{x}\left(\frac{z(x)}{2}|\nabla u(x)|^{2}+v(z(x))\right)\right)\label{eq:scale-mixture-model}
\end{equation}
with respect to some reference measure $\lambda^{n}(\mathrm{d}u)\otimes\bigotimes_{x}q(\mathrm{d}z(x))$.
Here, $\lambda^{N}$ denotes the $N$-dimensional Lebesgue measure and $q$ is a $1$-dimensional Borel measure on $[0,\infty)$, for example the Lebesgue
measure on $[0,\infty)$ or the counting measure on some discrete
subset of $[0,\infty)$.

Intuitively, the introduction of such latent scales allows the model
to freely ``choose'' the $z(x)$ in the most appropriate way. In particular,
the model can ``choose'' to make the $z(x)$ small on edges, if this
increases the overall probability. When done right, this can fix the
inability of the Gaussian model to preserve edges.

Given a noisy measurement $v_{n}$ of $Au$ with Gaussian noise,
the posterior density $p(u,z\mid v_{n})$ is given by
\begin{multline}
p(u,z\mid v_{n})\propto\exp\Bigl(-\frac{1}{2\sigma^{2}}\|Au-v_{n}\|^{2} \\
-\frac{1}{2}\sum_{x}z(x)|\nabla u(x)|^{2}-\sum_{x}v(z(x))\Bigr).\label{eq:gsm-posterior}
\end{multline}
Let
\begin{equation*}
\mu(\mathrm{d}u)=\lambda^{n}(\mathrm{d}u)\quad\text{and}\quad\nu(\mathrm{d}z)=\bigotimes_{x\in\Omega}q(\mathrm{d}z(x)).
\end{equation*}
The density in \eqref{eq:gsm-posterior} with respect to $\mu\otimes\nu$
is of the form
\begin{equation}
\label{eq:exponential-pair}
p(u,z\mid v_{n})\propto\exp\left(\langle s_{0}(u),r_{0}(z)\rangle+h(u)+g(z)\right)
\end{equation}
with so-called \emph{sufficient statistics}
\begin{equation}\label{eq:suffstatt-s0-r0}
s_{0}(u)=\left(-\frac{1}{2}|\nabla u(x)|^{2}\right)_{x\in\Omega}\quad\text{and}\quad r_{0}(z)=z
\end{equation}
and
\begin{equation}\label{eq:suffstat-h-g}
h(u)=-\frac{1}{2\sigma^{2}}\|Au-v_{n}\|^{2}\quad\text{and}\quad g(z)=-\sum_{x}v(z(x)).
\end{equation}
We call a pair of random variables $(u,z)$ that has a density of the form
\eqref{eq:exponential-pair} with respect to some reference measure $\mu \otimes \nu$
an \emph{exponential pair of random variables}. 
The name \emph{exponential pair} is inspired by the fact that any exponential pair,
both conditional densities $p(u \mid z, v_n)$ and $p(z \mid u, v_n)$ are in the
exponential family, i.e. they are of the form
$x\mapsto \mathrm e^{\theta^\intercal \bar s(x) - \bar A(\theta)}$,
see \cite[Section 9.2]{murphy2012machine} for details.
Using $s=(s_{0},h,1)$ and $r=(r_{0},1,g)$, the density of an exponential pair with respect to
$\mu \otimes \nu$ can also be written in the form
\begin{equation}
 p(u,z\mid v_{n})\propto\exp\left(\langle s(u),r(z)\rangle\right).
\end{equation}
We will use the notion of an exponential pair later in Section~\ref{sec:meanfield} to derive mean field approximations.

To see how the GSM in~(\ref{eq:gsm-posterior}) is related to the Perona-Malik model, we introduce
\begin{equation}
\psi(t)=-\log\int\mathrm{e}^{-tz-v(z)}q(\mathrm{d}z).\label{eq:scale-mixture-func}
\end{equation}
and  marginalize out the $z(x)$ in~(\ref{eq:gsm-posterior}): A simple calculation shows that
\begin{multline*}
  \int \exp\Big(\sum_{x}-\tfrac12 z(x)|\nabla u(x)|^{2}-v(z(x))\Big)q(dz) \\
  = \prod_{x}\int \exp\Big(-\tfrac12 z(x)|\nabla u(x)|^{2}-v(z(x))\Big)q(dz(x))\\
  = \prod_{x}\exp(-\psi(\tfrac12|\nabla u(x)|^{2})\Big) \\
  = \exp\Big(-\sum_{x}\psi(\tfrac12|\nabla u(x)|^{2})\Big)
\end{multline*}
and we obtain
\begin{multline}
\label{eq:scale-mixture-primal}
p(u\mid v_{n}) \propto
\exp\Biggl(-\frac{1}{2\sigma^{2}}\|Au-v_{n}\|^{2} \\
-\sum_{x}\psi\left(\frac{1}{2}|\nabla u(x)|^{2}\right)\Bigr).
\end{multline}
We see that the squared magnitude of the gradient is no longer penalized
linearly, but rescaled according to $\psi$ which is precisely the variational formulation~(\ref{eq:var-formulation}) of the Perona-Malik model.

In what follows, the derivatives of $\psi$ will play an important
role. The following lemma is an easy consequence of an analogous result from the theory of exponential families:

\begin{lemma}
\label{lem:psi-derivatives}$\psi$ is smooth on $(0,\infty)$ with
first and second derivative
\begin{enumerate}
\item $\psi^{\prime}(t)=\E\left(z(x)\mid\frac{1}{2}|\nabla u(x)|^{2}=t\right)$
\item $\psi^{\prime\prime}(t)=-\Var\left(z(x)\mid\frac{1}{2}|\nabla u(x)|^{2}=t\right)$.
\end{enumerate}
In particular, $\psi$ is a concave function.
\end{lemma}
\begin{proof}
  By logarithmic differentiation and looking at~(\ref{eq:scale-mixture-model}), we calculate the derivative of $\psi$ as
  \begin{align*}
    \psi'(t) & = \frac{\int z\mathrm{e}^{-tz-v(z)}q(\mathrm{d}z)}{\int\mathrm{e}^{-tz-v(z)}q(\mathrm{d}z)}\\
    & = \E\left(z(x)\;\Big|\; \frac{1}{2}|\nabla u(x)|^{2}=t\right).
  \end{align*}
  For the second claim note that (using 1.)
  \begin{multline*}
    \Var\left(z(x)\mid \frac12|\nabla u(x)|^{2}=t\right) \\ = 
   \E\left(z(x)^2\;\Big|\; \frac{1}{2}|\nabla u(x)|^{2}=t\right) - \psi^\prime(t)^2
  \end{multline*}
  which, after some manipulation, evaluates to $-\psi^{\prime\prime}(t)$.\qed
\end{proof}

The fact that $\psi$ is a concave function has the important consequence
that computing a MAP-assignment of \eqref{eq:scale-mixture-primal}
generally leads to non-convex optimization problems.

Moreover, Lemma \ref{lem:psi-derivatives} implies that we can compute
the first and second order moments of $z$ given $u$ without having
an explicit representation of the joint distribution $p(u,z)$. All
we need is an explicit representation of $\psi$ and knowledge that
such an explicit representation exists.
Note that the function $\Psi(t) = \mathrm{e}^{-\psi(t)}$ is in fact the Laplace
transform of the measure $\mathrm{e}^{-v}q$.
By Bernstein's theorem, we know that $\psi$ can be written in the form~(\ref{eq:scale-mixture-func}) if an only if $\Psi(t) = \mathrm{e}^{-\psi(t)}$
defines a completely monotone function~\cite{andrews1974scale}, i.e. for $k\geq 0$ it holds that $(-1)^{k}\Psi^{(k)}\geq 0$.

Using the particularly simple $v(z) = z$ and $q$ being the Lebesgue measure on $(0,\infty)$ we get
\[
\psi(t) = -\log\int_{0}^{\infty}e^{-tz-z}dz = -\log\left(\frac1{1+t}\right) = \log(1+t).
\]
Using this $\psi$ in~(\ref{eq:scale-mixture-primal}) we obtain the respective minimization problem~(\ref{eq:var-formulation}) with $F=\psi$ and hence, the differential equation~(\ref{eq:pde-formulation}) is
\begin{align*}
  \partial_{t}u + A'(Au-u_{n}) & = \nabla\cdot\Big(\psi^\prime\big(\tfrac12 |\nabla u|^{2}\big)\nabla u\Big)\\
   & = \nabla\cdot\Bigg(\frac{\nabla u}{1 + \tfrac12|\nabla u|^{2}}\Bigg)
\end{align*}
which is the original Perona-Malik model~\cite{perona1990scalespace}.
Note that Lemma~\ref{lem:psi-derivatives} provides a new interpretation for the diffusion coefficient:
it says that the diffusion coefficient $\psi^\prime$ is the expectation of the latent variable $z$ conditioned on the current magnitude of the image gradient.
The other proposed function from~\cite{perona1990scalespace}, $\psi^\prime(t) = \exp(-t)$ leads to $\psi(t) = 1-\exp(-t)$. This function also fits into our framework,
since the mapping $\Psi(t) =  \exp(-\psi(t)) = \exp(\exp(-t)-1)$ can be shown to be completely monotone.%
\footnote{By a simple induction argument one gets $\Psi^{(n)}(t) = (-1)^{n}\sum_{k=0}^{n-1}\binom{n}{k}(-1)^{k}\Psi^{(k)}(t)\exp(-t)$ from which we obtain $(-1)^{n}\Psi^{(n)}\geq 0$ as desired.}
However, we are not aware of a closed form for the inverse Laplace transform of this function, and hence, the form of $v$ remains elusive in this case.


\section{Methods}
\label{sec:methods}

In this section we develop different approaches to use the posterior density~(\ref{eq:gsm-posterior}) for image reconstruction tasks, such as denoising or deblurring. 

\subsection{MAP estimation}
\label{sec:em}

The first method we derive solves the maximum a-posteriori (MAP) problem for Gaussian scale
mixtures like in \eqref{eq:scale-mixture-model}, i.e. we consider
the posterior distribution over $u$ and $z$, which is given by
\begin{multline}
\label{eq:gsm-map-posterior}
p(u,z\mid v_{n})\propto
\exp\Bigl(-\frac{1}{2\sigma^{2}}\|Au-v_{n}\|^{2} \\
-\sum_{x}\left(\frac{1}{2}z(x)|\nabla u(x)|^{2}+v(z(x))\right)\Bigr)
\end{multline}
and aim for a maximizer of the distribution.

As derived in the previous section, the marginal distribution with respect to $u$ is
\begin{multline}
\label{eq:map-marginal-model}
p(u\mid v_{n})
\propto\exp\Biggl(-\frac{1}{2\sigma^{2}}\|Au-v_{n}\|^{2} \\
-\sum_{x}\psi\Bigl(\frac{1}{2}|\nabla u(x)|^{2}\Bigr)\Biggr).
\end{multline}
There are two different versions of MAP-assignments that can be computed: we
can either compute the MAP of both $u$ and $z$ with respect to the
joint distribution $p(u,z\mid v_{n})$ or the MAP of either $u$ or
$z$ with respect to the marginal distributions $p(u\mid v_{n})$
and $p(z\mid v_{n})$.

Both methods can be understood as methods of approximating the joint
distribution $p(u,z\mid v_{n})$ by a simpler, possibly deterministic,
probability distribution. However, computing the MAP over the joint
distribution is generally a crude approximation, as it approximates
$p(u,z\mid v_{n})$ with a completely deterministic distribution.
In contrast to this, computing the MAP with respect to $p(u\mid v_{n})$
still captures the uncertainty in $z$.
An even better alternative is to use mean field approximations which we
will discuss in Section~\ref{sec:meanfield}.

In this paper we consider the problem of computing a MAP-assign\-ment with respect
to the marginal model in \eqref{eq:map-marginal-model}. The corresponding optimization
problem can be written as
\begin{equation}
\label{eq:gsm-marginal-map-problem}
\min_{u}\frac{1}{2\sigma^{2}}\|Au-v_{n}\|^{2}+\sum_{x}\psi\left(\frac{1}{2}|\nabla u(x)|^{2}\right).
\end{equation}
In general, there are multiple optimization methods that can be employed
and in some cases there are specialized methods to solve the optimization
problem. For example, if $t\mapsto\psi\left(t^{2}\right)$ defines
a convex function, we can use tools from convex optimization to solve
the optimization problem. However, this is not the case for general
$\psi$, so that it makes sense to look for general purpose algorithms.

A simple option is to use gradient descent on \eqref{eq:gsm-marginal-map-problem}  which formally leads to the iteration
\begin{multline*}
  u^{k+1} = u^{k} \\- h_{k}\Big(\tfrac1{\sigma^{2}}A'(Au^{k}-v_{n}) - \nabla\cdot\big(\psi'(\tfrac12|\nabla u^{k}|^{2})\nabla u^{k})\big)\Big).
\end{multline*}
For suitable stepsizes $h_{k}$ this is a descent method, but the choice of stepsize is cumbersome and convergence is usually slow.
A second option is to use EM as we describe later in this section.
Before we do so, we would like to describe another way: As $-\psi$ is a convex function (which we extend by $+\infty$ for negative arguments), it is also possible to
dualize $-\psi$ as
\begin{equation*}
-\psi(t)=\max_{s\in\mathbb{R}}\Big\{s\,t-(-\psi)^{*}(s)\Big\}
\end{equation*}
which we also write as
\begin{equation*}
  \psi(t) = \min_{s\in\mathbb{R}}\Big\{-s\,t + (-\psi)^{*}(s)\Big\}.
\end{equation*}
Hence, we can reformulate~\eqref{eq:gsm-marginal-map-problem} as
\begin{multline*}
\min_{u,s}\frac{1}{2\sigma^{2}}\|Au-v_{n}\|^{2} + \sum_{x}\left( -s(x)\tfrac12|\nabla u(x)|^{2} \right.\\ + \left. (-\psi)^{*}\left( s(x)\right) \right).
\end{multline*}
Since the problem is convex in both $u$ and $s$ (but not jointly so), we can perform coordinate descent. This leads to the following updates:
\begin{align*}
  s^{k+1}& = -\psi'(\tfrac12|\nabla u^{k}|^{2})\\
  u^{k+1}& \leftarrow \text{solve}\ \tfrac{1}{\sigma^2} A'(Au-v_{n})+\nabla\cdot(s^{k+1}\nabla u) = 0\\
\end{align*}
Here we used that the inverse of the derivative of $(-\psi)^{*}$ is $(-\psi)'$. We can combine this into one iteration and get that $u^{k+1}$ is given as the solution of the linear equation
\begin{equation*}
  \tfrac{1}{\sigma^2}A'(Au^{k+1}-v_{n})-\nabla\cdot(\psi'(\tfrac12|\nabla u^{k}|^{2})\nabla u^{k+1}) = 0.
\end{equation*}
In the context of the Perona-Malik equation, this scheme is known as \emph{lagged diffusivity}~\cite{vogel1996iterative,chan1999convergence}. From the above derivation we obtain a new proof of that fact that this scheme is indeed a descent method for the objective function in~\eqref{eq:gsm-marginal-map-problem}.

Next we derive an EM method and we will see that EM and the coordinate descent method are equivalent.

The EM method alternates between an E step (\emph{expectation}) and an M step (\emph{maximization}).
In this particular example, we alternatingly estimate the $z$-variable and maximize with respect to the $u$ variable.
The E step for $z$ is, as derived in Lemma~\ref{lem:psi-derivatives}
\[
\xi_{0}^{(k)} \leftarrow \E(z(x)\mid \tfrac12|\nabla u^{(k)}(x)|^{2}) = \psi'(\tfrac12|\nabla u^{(k)}(x)|^{2})
\]
and the M step is to maximize the expectation
\[
 \E_{z} \left( \log p(u,z\mid v_{n}) \mid u^{(k)} \right)
\]
with respect to $u$.

Together, this leads the scheme in Algorithm~\ref{alg:gsm-map}.

\begin{algorithm}
  \noindent \begin{centering}
    \begin{algorithmic}[1]
    \Procedure{gsm\_map}{$v_n$}
        \While{not converged}
            \For{$x \in \Omega$}
                \State $\xi_0(x) \gets
                    \psi'(\tfrac{1}{2}|\nabla u(x)|^2)$
            \EndFor

            \State $u \gets
               \argmin_{u'}
               \frac{1}{2\sigma^2}\|Au' - v_n\|^2
               + \frac{1}{2}\sum_x\xi_0(x) |\nabla u'(x)|^2$

        \EndWhile
        \State \textbf{return} $u$
    \EndProcedure
\end{algorithmic}

    \par\end{centering}

  \caption{EM algorithm for Gaussian scale mixture\label{alg:gsm-map}}
\end{algorithm}

On the other hand, the gradient of the objective function in~\eqref{eq:gsm-marginal-map-problem} is given by
\begin{equation*}
\frac{1}{\sigma^{2}}A'(Au-v_{n})-\nabla\cdot\left(\psi^{\prime}\left(\frac{1}{2}|\nabla u|^{2}\right)\nabla u\right)
\end{equation*}
and gradient descent corresponds to solving the gradient flow
\begin{equation*}
\partial_{t}u+\frac{1}{\sigma^{2}}A'(Au-v_{n})=\nabla\cdot\left(\psi^{\prime}\left(\frac{1}{2}|\nabla u|^{2}\right)\nabla u\right).
\end{equation*}

Instead of finding
stationary points of this equation by integrating the differential
equation, this can be done more efficiently by using
lagged diffusivity \cite{vogel1996iterative,chan1999convergence}, which we described above.
If we set
\begin{equation}
\xi_{0}^{(k)}:=\psi^{\prime}\left(\frac{1}{2}|\nabla u^{(k)}|^{2}\right)\label{eq:lagged-diffusivity-parameter-update}
\end{equation}
we have to solve the linear equation
\begin{equation}
\frac{1}{\sigma^{2}}A'(Au-v_{n})-\nabla\cdot\left(\xi_{0}^{(k)}\nabla u\right)=0\label{eq:lagged-diffusivity}
\end{equation}
for $u$ to obtain the next iterate $u^{(k+1)}$. As it turns out, for Gaussian scale mixtures, this is equivalent to EM:
\begin{lemma}
\label{lem:EM-lagged-diffusivity-equivalence}
For Gaussian scale mixture
models as in \eqref{eq:scale-mixture-model}, EM and lagged diffusivity
yield the same algorithm.
\end{lemma}
\begin{proof}
By Lemma \ref{lem:psi-derivatives}, the updates of $\xi_{0}^{(k)}$
can be expressed as
\begin{equation*}
\xi_{0}^{(k)}=\psi^{\prime}\left(\frac{1}{2}|\nabla u^{(k)}|^{2}\right)=\E(z\mid u^{(k)}).
\end{equation*}

Moreover, solving \eqref{eq:lagged-diffusivity} is equivalent to
minimizing
\begin{equation*}
\frac{1}{2\sigma^{2}}\|Au-v_{n}\|^{2}+\frac{1}{2}\sum_{x}\xi_{0}^{(k)}(x)|\nabla u(x)|^{2}
\end{equation*}
with respect to $u$. Overall, we obtain $u^{(k+1)}$ from $u^{(k)}$
by minimizing
\begin{equation*}
\frac{1}{2\sigma^{2}}\|Au-v_{n}\|^{2}+\frac{1}{2}\sum_{x}\E(z(x)\mid u^{(k)})|\nabla u(x)|^{2}
\end{equation*}
with respect to $u$, which is just the EM algorithm.\qed
\end{proof}
Lemma \ref{lem:EM-lagged-diffusivity-equivalence} implies that we
can apply the convergence theory for EM to the lagged-diffusivity
algorithm and vice versa.


\subsection{Mean field approximations}
\label{sec:meanfield}
MAP-assignments often yield non-representative samples of the posterior
distribution~\cite{kaipio2005statistical}. A better alternative is given by mean field theory.
Moreover, mean field approximation are often a good compromise
between MAP-assignments that neglect uncertainties in the model and
the conditional mean that can be problematic in a multimodal setting.

In this section, we describe how mean field theory can be applied
to Gaussian scale mixtures. 

The idea is to approximate the complicated distribution $p(u,z \mid v_n)$ by a factorized distribution $q_1(u)q_2(z)$
as good as possible in some sense.
The mean field approximation defines this sense as nearness in the Kullback-Leibler divergence.\footnote{For two probability distributions $p$ and $q$, the Kullback-Leibler divergence is $\KL(q,p) = \int q\log\Big(\tfrac{q}{p}\Big)$.}
Hence, we denote with $p$ the density $p(u,z\mid v_{n})$ and seek distributions $q_{1}$ and $q_{2}$ such that
\begin{equation*}
  \KL(q_{1}q_{2},p)  = \iint q_{1}(u)q_{2}(z)\log\left(\frac{q_{1}(u)q_{2}(z)}{p(u,z | v_n)}\right)\mathrm{d}u\mathrm{d}z
\end{equation*}
is minimized.
We rewrite this term with the help of the entropy $\entropy(q) = -\int q\log q = -\E(\log(q))$ as
\begin{equation*}
\KL(q_{1}q_{2},p)  = -\langle q_{1}q_{2},\log p\rangle - \entropy(q_{1}) - \entropy(q_{2}).
\end{equation*}
Performing the minimization only over probability distributions $q_{1}$, we see that
\[
\int \log(p(u,z | v_n))q_{2}(z)\mathrm{d}z - \log(q_{1}(u)) = \const
\]
and for $q_{2}$ we get
\[
\int \log(p(u,z | v_n))q_{1}(u)\mathrm{d}u - \log(q_{2}(z)) = \const.
\]
Hence, we see, that alternating minimization for $q_{1}$ and $q_{2}$ leads to the updates
\begin{align}
 q_1^{(k+1)}(u) &\propto \exp\Bigl(\int\log(p(u,z| v_n))q_{2}^{(k)}(z)\mathrm{d}z\Bigr) \label{eq:meanfield-u}\\
  & = \exp\left(\E_{q_2^{(k)}(z)} (\log p(u, z| v_n)) \right) \notag\\
  q_2^{(k+1)}(z) &\propto \exp\Bigl(\int\log(p(u,z| v_n))q_{1}^{(k+1)}(u)\mathrm{d}u\Bigr)\label{eq:meanfield-z}\\
  & = \exp\left(\E_{q_1^{(k+1)}(u)} (\log p(u, z | v_n))\right). \notag
\end{align}
Note the resemblance of this procedure to the EM-algorithm.
In contrast to the EM-algorithm, however, the mean field algorithm treats both $u$ and $z$ symmetrically and incorporates the
uncertainty in both of them. The EM-algorithm ignores the uncertainty in $u$.

For Gaussian scale mixtures of the form~\eqref{eq:gsm-posterior} we have
\begin{multline*}
\log p(u,z | v_n) = -\frac{1}{2\sigma^{2}}\|Au-v_{n}\|^{2}  \\
-\sum_{x}\left(\frac{1}{2}z(x)|\nabla u(x)|^{2}+v(z(x))\right) + const 
\end{multline*}
The mean field-updates \eqref{eq:meanfield-u} and \eqref{eq:meanfield-z} now read
\begin{align}
 q_1^{(k+1)}(u) & \propto \exp\Biggl( -\frac{1}{2\sigma^{2}}\|Au-v_{n}\|^{2} \label{eq:meanfield-q1}\\
 & - \sum_{x}\left(\frac{1}{2}\xi_0^{(k)}(x)|\nabla u(x)|^{2})\right)\Biggr) \notag\\
 q_2^{(k+1)}(z) &\propto \exp\Biggl(
 \sum_{x}\left(\frac{1}{2}z(x) \eta_0^{(k+1)}(x) + v(z(x))\right)\Biggr), \label{eq:meanfield-q2}
\end{align}
where
\begin{align*}
\xi_{0}^{(k)}(x) & = \E_{q_2^{(k)}(z)}( z(x) )  \\
\eta_{0}^{(k)}(x) & =\E_{q_1^{(k)}(u)}\left(-\tfrac{1}{2}|\nabla u(x)|^{2}\right).
\end{align*}
Note that $q_1^{(k+1)}$ and $q_2^{(k+1)}$ in \eqref{eq:meanfield-q1} and \eqref{eq:meanfield-q2} can be regarded as (generalized) conditional distributions\footnote{
It is possible $\xi_0^{(k)}$ and $\eta_0^{(k+1)}$ are outside the range of $z$ and $-\frac{1}{2} |\nabla u|^2$ (e.g. if $z$ is a binary random variable). However, formally, $q_1^{(k+1)}$ and $q_2^{(k+1)}$ still behave like the indicated conditional distributions.}
\begin{equation*}
 p\left(u \;\left|\; z = \xi_0^{(k)}\right.\right) \hfill \text{and} \hfill p\left(z \;\left|\; -\frac{1}{2} |\nabla u|^2 = \eta_0^{(k+1)} \right.\right)
\end{equation*}
respectively.
Using Lemma \ref{lem:psi-derivatives}, this shows that
\begin{align}
\xi_0^{(k+1)}(x) & = \E\left( z(x) \mid -\tfrac{1}{2}|\nabla u(x)|^{2} = \eta_0^{(k)}(x)\right) \label{eq:gsm-mean-fieldupdates-xi} \\
& = \psi^\prime(-\eta_0^{(k)}(x)) \notag \\
\eta_0^{(k+1)}(x) & =\E \left(-\tfrac{1}{2}|\nabla u(x)|^{2} \mid z(x) = \xi_0^{(k+1)}(x)\right). \label{eq:gsm-mean-fieldupdates-eta}
\end{align}
The variance of a random vector $v$ is defined as
\begin{equation*}
\Var(v):=\E|v-\E v|^{2}=\E|v|^{2}-|\E v|^{2}.
\end{equation*}
Hence, we can write
\begin{multline*}
\E\left(-\tfrac{1}{2}|\nabla u(x)|^{2}\mid z(x)=\xi_{0}^{(k)}(x)\right) \\
= -\tfrac12\Bigl[|\E\big(\nabla u(x)\mid z(x) = \xi_{0}^{(k)}(x)\big)|^{2}\\
\quad+ \Var\big(\nabla u(x)\mid z(x) = \xi_{0}^{(k)}(x)\big)\Bigr].
\end{multline*}
Setting
\begin{equation*}
 u^{(k)}=\E(u\mid z=\xi_{0}^{(k)}),\quad
 \delta^{(k)}=\Var(\nabla u(x)\mid z=\xi_{0}^{(k)}),
\end{equation*}
we see that equation \eqref{eq:gsm-mean-fieldupdates-eta} can be written as
\begin{equation*}
\eta_{0}^{(k)}=-\tfrac{1}{2}\left(|\nabla u^{(k)}|^{2}+\delta^{(k)}\right).
\end{equation*}
Combining this with \eqref{eq:gsm-mean-fieldupdates-xi}, we see that
we can compute $u^{(k+1)}$ by minimizing
\begin{multline}
\frac{1}{2\sigma^{2}}\|Au-v_{n}\|^{2} \\
+\sum_{x}\psi'\left(\frac{1}{2}\left(|\nabla u^{(k)}(x)|^{2}+\delta^{(k)}(x)\right)\right)|\nabla u(x)|^{2}
\end{multline}
with respect to $u$. We see that mean field approximations to the
joint distribution $p(u,z\mid v_{n})$ as stated in Algorithm~\ref{alg:gsm-meanfield} yield a modified version of
lagged diffusivity, where the modification is given by
\begin{equation*}
\delta^{(k)}(x)=\Var(\nabla u(x)\mid z(x) = \xi_{0}^{(k)}).
\end{equation*}
This modification captures the uncertainty in $u$ that we have
in the model.

Unfortunately, $\delta^{(k)}$ is hard to compute in practice. The
most straightforward way of computing it requires the full covariance
matrix of $u$ which is intractable in high dimensions. Another approach
would be to compute $\delta^{(k)}$ using a Monte Carlo approach.
While this only requires sampling from a Gaussian distribution and
can therefore be performed using perturbations sampling \cite{papandreou2010gaussian}, Monte Carlo
methods are slow to converge and we have to solve a linear equation
several times in each iteration.

We therefore propose an approximate procedure to compute the $\delta^{(k)}$
which is based on a relaxation of the mean-field optimization problem.

\begin{algorithm}
\noindent \begin{centering}
\begin{algorithmic}[1]
    \Procedure{gsm\_meanfield}{$v_n$}
        \While{not converged}
            \For{$x \in \Omega$}
                \State $\delta(x) \gets
                    \textsc{compute\_delta}(\xi_0)$
                \State $\xi_0(x) \gets
                    \psi^\prime\left(\tfrac{1}{2}\left(|\nabla u(x)|^2 + \delta(x)\right)\right)$
            \EndFor

            \State $u \gets
               \argmin_{u'}
               \frac{1}{2\sigma^2}\|Au' - v_n\|^2
               + \frac{1}{2}\sum_x\xi_0(x) |\nabla u'(x)|^2$

        \EndWhile
        \State \textbf{return} $u$
    \EndProcedure
\end{algorithmic}

\par\end{centering}

\caption{Mean field algorithm for Gaussian scale mixture.\label{alg:gsm-meanfield}}
\end{algorithm}

In the following we derive and explicit representation of the mean field objective in the general case of distributions that are exponential pairs. 
Recall, that two random variables $u$ and $z$ form an exponential pair if their distribution can be written as
\begin{equation}
  p(u,z) = \exp\left(\langle s(u), r(z) \rangle \right).
\end{equation}
We define the functions
\begin{align}
  H(\xi) & := \log \int \exp\left( \langle s(u), \xi \rangle \right) \mu(\mathrm d u) \\
  G(\eta) & := \log \int \exp\left( \langle \eta, r(z) \rangle \right) \nu(\mathrm d z)
\end{align}
which are closely related to the so-called log-partition function from the theory of exponential families
(cf.~\cite[Section 9.2]{murphy2012machine}).
These functions allow us to express the marginals and conditional densities of $p$ as
\begin{align*}
  p(u) & = \int p(u,z)\nu(\mathrm{d}z) = \int \exp(\langle s(u),r(z)\rangle)\nu(\mathrm{d}z)\\
  & = \exp(G(s(u))).\\
  p(z) & = \int p(u,z)\mu(\mathrm{d}u) = \int \exp(\langle s(u),r(z)\rangle)\mu(\mathrm{d}u)\\
  & = \exp(H(r(z)))
\end{align*}
and
\begin{align*}
  p(u\mid z) & = \tfrac{p(u,z)}{p(z)} = \exp\Big(\langle s(u),r(z)\rangle - H(r(z))\Big)\\
  p(z\mid u) & = \tfrac{p(u,z)}{p(u)} = \exp\Big(\langle s(u),r(z)\rangle - G(s(u))\Big),
\end{align*}
respectively. 
Here the connection to the exponential family can be seen clearly: Both $p(u\mid z)$ and $p(z\mid u)$ are in the exponential family, the vector $s(u)$ contains the sufficient statistics for $u$, $r(z)$ is the respective parameter vector and vice versa for $z$. 
Note moreover that the density $p(u\mid z)$ is completely determined by the value $r(z)$and $p(z\mid u)$ is determined by $s(u)$ (and not by $z$ and $u$), respectively. Hence, we form
\begin{align*}
  p(u\mid \xi) & := \exp\Big(\langle s(u),\xi\rangle - H(\xi)\Big)\\
  p(z\mid \eta) & := \exp\Big(\langle \eta,r(z)\rangle - G(\eta)\Big),
\end{align*}
Similarly as for the case of the log-partition function one can show that both $G$ and $H$ are convex functions (one shows that the Hessians of $G$ and $H$ are the covariance matrices of the respective sufficient statistics in the same way as done in~\cite[Section 9.2.3]{murphy2012machine}).
Hence, we can consider their convex conjugates, i.e.
\begin{align*}
  H^{*}(\eta) &= \sup_{\xi}\langle \xi,\eta\rangle - H(\xi)\\
  G^{*}(\xi) &= \sup_{\eta}\langle \eta,\xi\rangle - G(\eta).
\end{align*}
We use these descriptions to derive a mean field approximation to $p(u,z)$. Indeed we have the following theorem:

\begin{theorem}
\label{thm:exppair-meanfield}
The naive mean field approximation to
$p$ is given by $q_{1}(u\mid\xi)q_{2}(z\mid\eta)$ where
\begin{align}
q_{1}(u\mid\xi) & =\mathrm{e}^{\langle s(u),\xi\rangle-H(\xi)}\label{eq:exppair-meanfield-sol-1}\\
q_{2}(z\mid\eta) & =\mathrm{e}^{\langle\eta,r(z)\rangle-G(\eta)}\label{eq:exppair-meanfield-sol-2}
\end{align}
The Kullback-Leibler divergence of $q_{1}(u\mid\xi)q_{2}(z\mid\eta)$ and $p(u,z)$ has the following explicit form
\begin{equation}
\KL(q_{1}(u\mid\xi)q_{2}(z\mid\eta),p(u,z))=H^{*}(\tilde{\xi})+G^{*}(\tilde{\eta})-\langle\tilde{\xi},\tilde{\eta}\rangle,\label{eq:KL-joint}
\end{equation}
where $\tilde{\xi}=\nabla H(\xi)$ and $\tilde{\eta}=\nabla G(\eta)$.
A point $(\xi,\eta)$ is a stationary point of the mean field objective
in \eqref{eq:KL-joint}, iff it satisfies
\[
\xi=\nabla G(\eta)\quad\text{and}\quad\eta=\nabla H(\xi).
\]
\end{theorem}
\begin{proof}
The proof uses the close relationship between exponential pairs and exponential families.
In particular, we are going to use that for the conjugates $H^{*}$ and $G^{*}$ we have the description
$H^{*}(\eta) = -\entropy(q_{1})$ and $G^{*}(\eta) = -\entropy(q_{2})$, cf.~\cite[Section 3.6]{wainwright2008graphical}.

A mean field
approximation of the form $\tilde{q}_{1}(u)\tilde{q}_{2}(z)$
to a distribution $p$ satisfies
\[
\tilde{q}_{1}(u)\propto\exp\left(\E(\log p(u,z)\mid\theta)\right)=\exp\left(\langle s(u),\E(r(\theta)\mid u)\rangle\right).
\]
This yields 
\[
\tilde{q}_{1}(u)=\mathrm{e}^{\langle\xi,s(u)\rangle-H(\xi)}
\]
with $\xi=\E(r(z)\mid u)$ and similarly for $\tilde{q}_{2}(z)$.
This also shows
\[
\begin{split}
  \xi &=\E(r(z)\mid u) =\nabla
  G(\eta)\quad\text{and}\\
  \eta&=\E(s(u)\mid\xi)=\nabla H(\xi).
\end{split}
\]

For any densities $q_{1}$ and $q_{2}$ like in \eqref{eq:exppair-meanfield-sol-1}and \eqref{eq:exppair-meanfield-sol-2} we calculate $\KL(q_{1}(u\mid\xi)q_{2}(z\mid\eta),p(u,z\mid v_{n}))$ as
\[
\entropy(q_{1}q_{2},p)-\entropy(q_{1})-\entropy(q_{2})=-\langle\tilde{\xi},\tilde{\eta}\rangle+H^{*}(\tilde{\xi})+G^{*}(\tilde{\eta}),
\]
because 
\[
\entropy(q_{1}q_{2},p)=\E_{q_{1}q_{2}}(-\log p(u,z))=-\langle\tilde{\xi},\tilde{\eta}\rangle,
\]
as well as $\entropy(q_{1})=-H^{*}(\tilde{\xi})$ and $\entropy(q_{2})=-G^{*}(\tilde{\eta})$. \qed
\end{proof}

Using $s=(s_{0},h,1)$ and $r=(r_{0},1,g)$ the mean field objective~\eqref{eq:KL-joint} is equivalent to
\begin{equation*}
\min_{\xi_{0},\xi_{1}}\min_{\eta_{0}\eta_{1}}H_{0}^{*}(\eta_{0},\eta_{2})+G_{0}^{*}(\xi_{0},\xi_{1})-\langle\eta_{0},\xi_{0}\rangle-\xi_{1}-\eta_{2}
\end{equation*}
with $H_{0}$ and $G_{0}$ defined as
\begin{align}
H_{0}(\xi_{0},\xi_{1}) & :=\log\int\mathrm{e}^{\langle s_{0}(u),\xi_{0}\rangle+\xi_{1}h(u)}\mu(\mathrm{d}u).\label{eq:exppair-conjugate-0}\\
G_{0}(\eta_{0},\eta_{2}) & :=\log\int\mathrm{e}^{\langle\eta_{0},r_{0}(z)\rangle+\eta_{2}g(z)}\nu(\mathrm{d}z).
\end{align}

The mean field problem
can be written in the alternative form
\begin{equation}\label{eq:min-G*-H}
\min_{\xi}G^{*}(\xi)-H(\xi).
\end{equation}

Using $s=(s_{0},h,1)$ and $r=(r_{0},1,g)$, this can be expressed
as
\begin{equation*}
\min_{\xi_{0},\xi_{1}}G_0^{*}(\xi_{0},\xi_{1})-H_0(\xi_{0},1)-\xi_{1}.
\end{equation*}

Now we apply the previous findings to the case where $s_{0}, r_{0}, g$ and $h$ are given by~\eqref{eq:suffstatt-s0-r0} and~\eqref{eq:suffstat-h-g} and derive an explicit description for $H_{0}(\xi_{0},1)$.
\begin{lemma}\label{lem:gsm-H-G}
We define the linear mapping $\Lambda(\xi_{0})$ by
\[
\Lambda(\xi_{0})u = \tfrac1{\sigma^{2}}A'Au - \nabla\cdot(\xi_{0}\nabla u)
\]
and set $m := \tfrac1{\sigma^{2}}A'v_{n}$. Then $\Lambda(\xi_{0})$
is positive definite if $A\mathbf{1}\neq 0$ and $\xi_{0}>0$. Furthermore it holds that
$H_{0}(\xi_{0},1)$ is given by
\begin{equation}
H_{0}(\xi_{0},1)=-\frac{1}{2}\log\det\Lambda(\xi_{0})+\frac{1}{2}\langle m,\Lambda(\xi_{0})^{-1}m\rangle+\const.\label{eq:gsm-H}
\end{equation}
\end{lemma}
\begin{proof}
Definiteness of $\Lambda$ follows by
\[
\langle u,\Lambda(\xi_{0})u\rangle = \tfrac1{\sigma^{2}}\norm{Au}^{2}+ \norm{\xi_{0}|\nabla u|}^{2}
\]
which it positive for non-zero $u$.

By~\eqref{eq:exppair-conjugate-0} we have
\begin{multline}
\label{eq:H-integral}
\exp\big(H_{0}(\xi_{0},1)\big) =
\int\exp\Biggl(-\frac{1}{2\sigma^{2}}\|Au-v_{n}\|^{2} \\
-\frac{1}{2}\sum_{x}\xi_{0}(x)|\nabla u(x)|^{2}\Biggr)\mathrm{d}u.
\end{multline}
A straightforward calculation shows that the integrand is proportional
to
\begin{multline}
\label{eq:H-integrand}
\exp\Biggl(-\frac{1}{2}\langle u-\Lambda(\xi_{0})^{-1}m,\Lambda(\xi_{0})(u-\Lambda(\xi_{0})^{-1}m)\rangle \\
+ \frac{1}{2}\langle m,\Lambda(\xi_{0})^{-1}m\rangle\Biggr).
\end{multline}
Hence, the integral on the right in~\eqref{eq:H-integral} is a Gaussian distribution and 
by a standard result about the normalization constant of a Gaussian
distribution, we see that \eqref{eq:H-integral} is proportional to
\begin{equation*}
\sqrt{\frac{(2\pi)^{N}}{\det|\Lambda(\xi_{0})|}}\mathrm{e}^{\frac{1}{2}\langle m,\Lambda(\xi_{0})^{-1}m\rangle}.
\end{equation*}
Taking the logarithm of this, results in equation \eqref{eq:gsm-H}.\qed
\end{proof}

Instead of computing the convex conjugate of $H_{0}$, we compute
the convex conjugate over the two terms in \eqref{eq:gsm-H} separately.
Recall that for $\Lambda$ positive definite
\begin{equation*}
-\frac{1}{2}\log\det\Lambda=\max_{C}-\frac{1}{2}\langle C,\Lambda\rangle+\frac{1}{2}\log\det C+\frac{N}{2}
\end{equation*}
where $C$ ranges over the set of positive semidefinite $N\times N$ matrices and $N$ the
total number of pixels. This shows that
\begin{multline}
H_{0}(\xi_{0},1)
=\max_{C,\eta_{0}}-\frac{1}{2}\langle C,\Lambda(\xi_{0})\rangle+\frac{1}{2}\log\det C+\langle\eta_{0},\xi_{0}\rangle \\
-F^{*}(\eta_{0})+\const.
\end{multline}
where
\begin{equation*}
F(\xi_{0})=\frac{1}{2}\langle m,\Lambda(\xi_{0})^{-1}m\rangle.
\end{equation*}
This turns the minimization problem~\eqref{eq:min-G*-H} into
\begin{multline}
\label{eq:mean-field-optimization-reformulated}
\min_{\xi_{0},\xi_{1}}\min_{C}\min_{\eta_{0}}
G_0^{*}(\xi_{0},\xi_{1})+F^{*}(\eta_{0}) \\
-\frac{1}{2}\log\det C+\frac{1}{2}\langle C,\Lambda(\xi_{0})\rangle-\langle\eta_{0},\xi_{0}\rangle-\xi_{1}.
\end{multline}
In principle, this optimization problem can again be solved by coordinate
descent. However, computation of $C$ is intractable, as it represents
a $N\times N$ matrix. We therefore replace the set of allowable $C$
by a lower dimensional set for which we can explicitly compute the
determinant. One such set is given by the set of diagonal matrices.
Even though we will no longer find a local optimum, we still minimize
an upper bound to \eqref{eq:mean-field-optimization-reformulated}.
We denote the diagonal entries of $C$ by $(c(y))_{y\in\Omega}$ and for some pixel $y\in\Omega$ we denote by $\delta_{y}$ the image which is one only in pixel $y$ and zero elsewhere.
Then the optimization problem in \eqref{eq:mean-field-optimization-reformulated}
becomes
\begin{multline}
\label{eq:relaxed-mean-field}
\min_{\xi_{0},\xi_{1}}\min_{c}\min_{\eta_{0}}
G_0^{*}(\xi_{0},\xi_{1})+F^{*}(\eta_{0}) \\
-\frac{1}{2}\sum_{y}\log c(y)+\frac{1}{2}\sum_{y}c(y)\langle \delta_{y},\Lambda(\xi_{0})\delta_{y}\rangle 
-\langle\eta_{0},\xi_{0}\rangle-\xi_{1}.
\end{multline}

Minimization of \eqref{eq:relaxed-mean-field} with respect to $\xi_{0}$
and $\xi_{1}$ yields
\begin{equation*}
\left(\eta_{0}-\frac{1}{2}\sum_{y}c(y)|\nabla\delta_{y}(x)|^{2},1\right)\in\partial G_0^{*}(\xi_{0},\xi_{1}).
\end{equation*}
Here, we used the fact that
\begin{equation*}
  \frac{\partial}{\partial\xi_{0}(x)}\Lambda(\xi_{0}) u(x)= -\nabla\cdot\nabla u(x) = -\Delta u(x).
\end{equation*}
and hence
\begin{align*}
  \frac{\partial}{\partial \xi_{0}(x)}\Big(\frac{1}{2}\sum_{y}c(y)\langle \delta_{y},\Lambda(\xi_{0})\delta_{y}\rangle\Big) &= \frac12\sum_{y}\langle \delta_{y},-\Delta \delta_{y}\rangle\\
 & = \frac12\sum_{y}|\nabla \delta_{y}|^{2}.
\end{align*}
Consequently, by subgradient inversion
\begin{align*}
\xi_{0} & =\nabla_{\eta_{0}}G_{0}\left(\eta_{0}-\frac{1}{2}\sum_{y}c(y)|\nabla\delta_{y}(x)|^{2},1\right)\\
\xi_{1} & =\partial_{\eta_{1}}G_{0}\left(\eta_{0}-\frac{1}{2}\sum_{y}c(y)|\nabla\delta_{y}(x)|^{2},1\right).
\end{align*}
Similarly, we obtain
\begin{align*}
\eta_{0} & =\nabla F(\xi_{0})=-\frac{1}{2}\left|\nabla\left(\Lambda^{-1}\left(\xi_{0}\right)m\right)\right|{}^{2}\\
c(x) & =\frac{1}{\langle \delta_{x},\Lambda(\xi_{0})\delta_{x}\rangle}=\frac{1}{\frac{1}{\sigma^{2}}|A\delta_{x}|^{2}+\sum_{y}\xi_{0}(y)|\nabla\delta_{x}(y)|^{2}}.
\end{align*}

Overall, applying coordinate descent yields
\begin{align*}
\xi_{0}^{(k+1)}(x) & =\psi^{\prime}\left(-\eta_{0}^{(k)}(x)+\frac{1}{2}\sum_{y}c^{(k)}(y)|\nabla\delta_{y}(x)|^{2}\right)\\
\eta_{0}^{(k+1)}(x) & =-\frac{1}{2}\left|\nabla\left(\Lambda^{-1}\left(\xi_{0}^{(k+1)}\right)m\right)(x)\right|{}^{2}\\
c^{(k+1)}(x) & =\frac{1}{\frac{1}{\sigma^{2}}|A\delta_{x}|^{2}+\sum_{y}\xi_{0}^{(k+1)}(y)|\nabla\delta_{x}(y)|^{2}}.
\end{align*}

The $\xi_{1}$-updates are given by
\begin{equation*}
\xi_{1}^{(k+1)}=\partial_{\eta_{1}}G_0(\eta_{0}^{(k+1)},1).
\end{equation*}
However, as the other updates do not depend on $\xi_{1}$, we can
leave them out.

A visualization of $|\nabla\delta_{x}(y)|^{2}$ as a function of $x$
and as a function of $y$ is shown in Figure \ref{fig:deltaxy}.

\begin{figure}
\noindent \begin{centering}
\subfloat[As a function of $x$]{\noindent \begin{centering}
\begin{tikzpicture}[scale=.8]

	\fill [gray!60!white](1,1) rectangle (2,2);
	\draw (1.5,1.5) node  {$+2$};

	\fill [gray!20!white](1,2) rectangle (2,3);
	\draw (1.5,2.5) node  {$+1$};

	\fill [gray!20!white](2,1) rectangle (3,2);
	\draw (2.5,1.5) node  {$+1$};
	
	\draw[step=1cm, black!50!gray, thin] (.5,.5) grid (3.5,3.5);
	\draw[->] (0.3,0.3) -- (3.8,0.3) node[below right] {$x_1$};
	\draw[->] (0.3,0.3) -- (0.3,3.8) node[above left] {$x_2$};;

\end{tikzpicture}
\par\end{centering}

}\subfloat[As a function of $y$]{\noindent \begin{centering}
\begin{tikzpicture}[scale=.8]

	\fill [gray!60!white](2,2) rectangle (3,3);
	\draw (2.5,2.5) node  {$+2$};

	\fill [gray!20!white](1,2) rectangle (2,3);
	\draw (1.5,2.5) node  {$+1$};

	\fill [gray!20!white](2,1) rectangle (3,2);
	\draw (2.5,1.5) node  {$+1$};
	
	\draw[step=1cm, black!50!gray, thin] (.5,.5) grid (3.5,3.5);
	\draw[->] (0.3,0.3) -- (3.8,0.3) node[below right] {$y_1$};
	\draw[->] (0.3,0.3) -- (0.3,3.8) node[above left] {$y_2$};;

\end{tikzpicture}
\par\end{centering}

}
\par\end{centering}

\caption[Visualization of $|\nabla\delta_{x}(y)|^{2}$]{Visualization of $|\nabla\delta_{x}(y)|^{2}$ as a function of $x$
and as a function of $y$. The other variable is located at the dark
gray square.\label{fig:deltaxy}}

\end{figure}
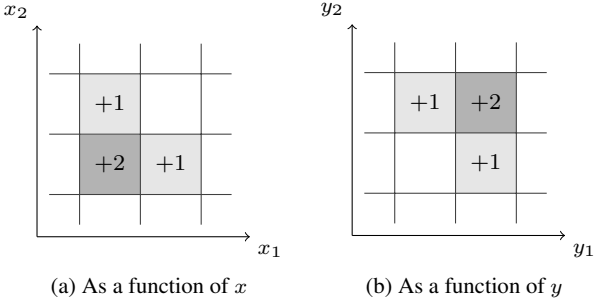

Note that $\Lambda^{-1}(\xi_{0}^{(k)})m$ is just the mean vector
$u^{(k)}$ of $p(u\mid\xi_{0}^{(k)},v_{n})$. Because $p(u\mid\xi_{0}^{(k)},v_{n})$ is a Gaussian distribution,
it is also the MAP-assignment. We therefore reobtain
Algorithm \ref{alg:gsm-meanfield} with
\begin{equation*}
\delta^{(k)}(x)=\sum_{y}c^{(k)}(y)|\nabla \delta_{y}(x)|^{2},
\end{equation*}
where
\begin{equation*}
 c^{(k)}(x)=\frac{1}{\frac{1}{\sigma^{2}}|A\delta_{x}|^{2}+\sum_{y}\xi_{0}^{(k)}(y)|\nabla\delta_{x}(y)|^{2}}.
\end{equation*}
The full algorithm is stated in Algorithm~\ref{alg:gsm-meanfield-approx}.

\begin{algorithm}
\noindent \begin{centering}
\begin{algorithmic}[1]
    \Procedure{gsm\_meanfield}{$v_n$}
        \State Initialize $c=0$ and $u=v_n$
        \While{not converged}
            \State $\delta(x) \gets
                \sum_{y}c(y)|\nabla\delta_{y}(x)|^{2}$
            \State $\xi_{0}(x) \gets \psi'\left(\tfrac12|\nabla u(x)|^{2} + \delta(x)\right)$
            \State $u \gets
               \argmin_{u'}
               \frac{1}{2\sigma^2}\|Au' - v_n\|^2
               + \frac{1}{2}\sum_x\xi_0(x) |\nabla u'(x)|^2$ 
	    \State $c(x)\gets \frac{1}{\tfrac1{\sigma^{2}}|A\delta_{x}|^{2} + \sum_{y}\xi_{0}(y)|\nabla\delta_{x}(y)|^{2}}$
        \EndWhile
        \State \textbf{return} $u$
    \EndProcedure
\end{algorithmic}


\par\end{centering}

\caption{Approximate mean field algorithm for Gaussian scale mixture.\label{alg:gsm-meanfield-approx}}
\end{algorithm}

Moreover, the $c^{(k)}$ can be interpreted as marginal variances.


\subsection{Sampling}
\label{sec:sampling}
As Gaussian scale mixtures are a special case of exponential pairs as defined in~(\ref{eq:exponential-pair}),
we can apply the efficient blockwise Gibbs sampler~\cite{geman1984stochastic}.
To this end, we need the conditional densities $p(u\mid z=\xi_{0},v_{n})$
and $p\left(z\mid-\tfrac{1}{2}|\nabla u|^{2}=\eta_{0},v_{n}\right)$.
\begin{lemma}
The conditional densities of $p$ in \eqref{eq:gsm-posterior} are
given by
\begin{align*}
p(u\mid z=\xi_{0})
& \propto\exp\Bigl(-\frac{1}{2\sigma^{2}}\|Au-v_{n}\|^{2} \\
& \qquad -\frac{1}{2}\sum_{x}\xi_{0}(x)|\nabla u(x)|^{2}\Bigr) \\
p\left(z\mid-\tfrac{1}{2}|\nabla u|^{2}=\eta_{0}\right)
& \propto\prod_{x}\exp\left(z(x)\eta_{0}(x)-v(z(x))\right).
\end{align*}
\end{lemma}

In particular, we see that $p(u\mid\xi_{0})$ is a Gaussian distribution
and $p\left(z\mid\eta_{0}\right)$ factors over the pixels $x$. This allows us to use
perturbation sampling \cite{papandreou2010gaussian}
to sample from $u$ given $z$. To sample $z$ given $u$, we can
just sample every component of $z$ individually.

Overall, we obtain Algorithm \ref{alg:gsm-sample} to sample from
a Gaussian scale mixture.
 Note that solving the optimization problem
\begin{equation*}
\arg\min_{u}\frac{1}{2\sigma^{2}}\|Au-v_{n}-\epsilon_{p}\|^{2}+\frac{1}{2}\sum_{x}z(x)|\nabla u(x)-\epsilon_{m}(x)|^{2}\ensuremath{}
\end{equation*}
is equivalent to solving the linear system
\begin{equation*}
\tfrac1{\sigma^{2}}A'(Au-v_{n}-\epsilon_{p})+\nabla\cdot\left(z(\nabla u-\epsilon_{m})\right)=0
\end{equation*}
which can be done efficiently, for example by using the cg-method
or a multigrid solver. 
Moreover, note the resemblance of the resulting algorithm
to \emph{lagged diffusivity}.

\begin{algorithm}
\noindent \begin{centering}
\begin{algorithmic}[1]
    \Procedure{gsm\_sample}{$v_n, N$}
        \For{$i=1,\cdots,N$}
            \For{$x \in \Omega$}
                \State $\eta_0(x) \gets -\tfrac{1}{2}|\nabla u(x)|^2$
                \State $z(x) \gets
                    \text{sample from } \exp(z(x)\eta_0(x) - v(z(x)   ))$
                \State $\epsilon_p(x) \gets
                    \text{sample from }
                    \mathcal N(0, \sigma^2)$
                \If{$z(x)\neq 0$}
                    \State $\epsilon_m(x) \gets
                        \text{sample from }
                        \mathcal{N}_{2}\left(0, \tfrac{1}{z(x)}\right)$
                \EndIf
            \EndFor

            \State $u \gets
               \argmin_{u'}
               \frac{1}{2\sigma^2}\|Au' - v_n - \epsilon_p\|^2$\par
        \hskip\algorithmicindent
               \hfill $+ \frac{1}{2}\sum_x z(x) |\nabla u'(x) - \epsilon_m(x)|^2$

        \EndFor
        \State \textbf{return} $u, z$
    \EndProcedure
\end{algorithmic}


\par\end{centering}

\caption{Sampling algorithm for Gaussian scale mixture\label{alg:gsm-sample}}
\end{algorithm}


\section{Applications}
\label{sec:applications}
In this section we show results of the methods derived in Section~\ref{sec:methods}.
All methods have been implemented in Julia~\cite{bezanson2014julia}.
We applied the methods to color images by applying the developed methods for all color channels but averaging the squared gradient magnitude over all channels such that all color channels use the same edge information.
The color range of the images is always $[0,1]^{3}$.

Figure \ref{fig:perona-malik-1-denoising} show the results of the
Gaussian scale mixture to a denoising problem. As the prior-function $v$ for the edge weights we used
\[
v(z) = \frac{z}{\lambda} - \Big(\frac{C}{\lambda}-1\Big)\log(z)
\]
with parameters $C,\lambda>0$. One gets 
\begin{align*}
  \int_{0}^{\infty}e^{-tz-v(z)} dz & = \int_{0}^{\infty}e^{-\big(t+\tfrac1\lambda\big)z}z^{\tfrac{C}{\lambda}-1}dz\\
  & = \frac{\Gamma\big(\tfrac{C}{\lambda}\big)}{\big(t+\tfrac1\lambda\big)^{\tfrac{C}{\lambda}}}
\end{align*}
and 
\[
\psi(t) = -\log\Big(\Gamma\big(\tfrac{C}{\lambda}\big)\Big) + \log\Big(\big(t+\tfrac1\lambda\big)^{\tfrac{C}{\lambda}}\Big).
\]
and hence, one obtains the Perona-Malik diffusivity
\[
f(t) = \psi'(t) = \frac{C}{1+\lambda t}.
\]

\begin{figure*}[hp]
\noindent \begin{centering}
\begin{tabular}{cc}
\subfloat[Uncorrupted image]{\noindent \begin{centering}
\includegraphics[width=0.4\textwidth]{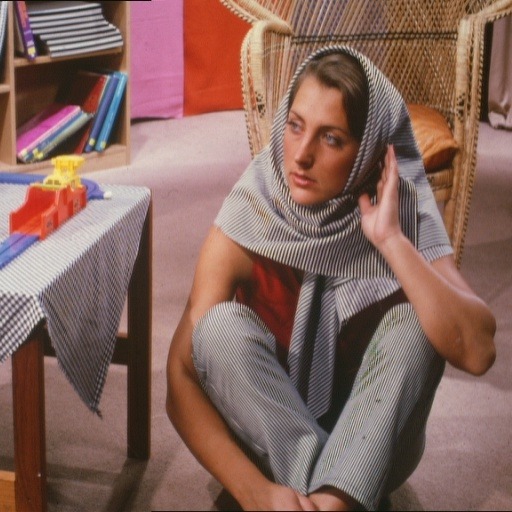}
\par\end{centering}

} & \subfloat[Noisy image\label{fig:perona-malik-denoising-corrupted}]{\noindent \begin{centering}
\includegraphics[width=0.4\textwidth]{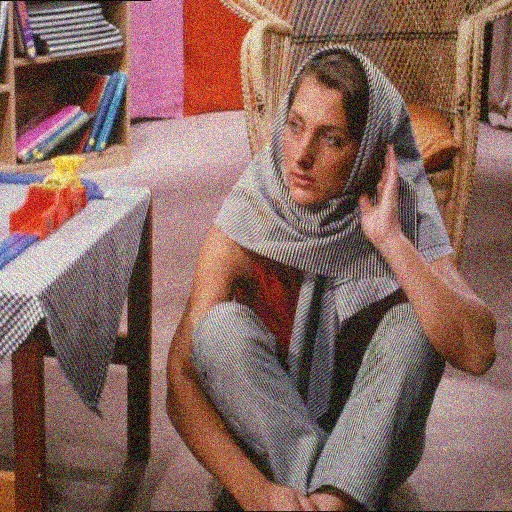}
\par\end{centering}

}\tabularnewline
\subfloat[Mean field approximation.\label{fig:perona-malik-denoising-meanfield}]{\noindent \begin{centering}
\includegraphics[width=0.4\textwidth]{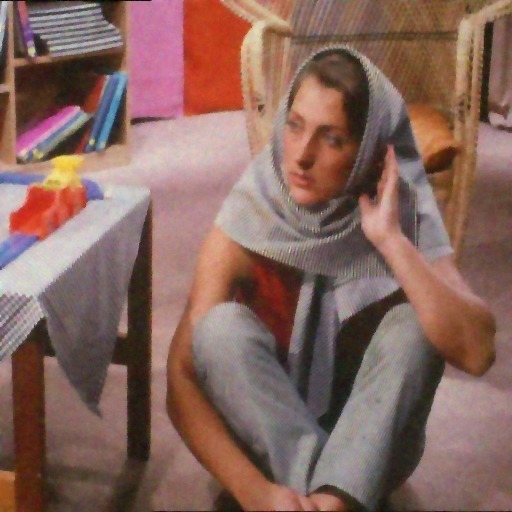}
\par\end{centering}

} & \subfloat[Mean edge image corresponding to mean field approximation\label{fig:perona-malik-denoising-meanfield-edge}]{\noindent \begin{centering}
\includegraphics[width=0.4\textwidth]{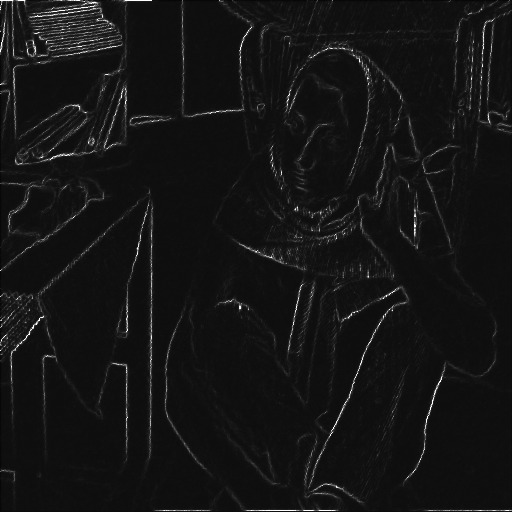}
\par\end{centering}

}\tabularnewline
\subfloat[MAP-assignment\label{fig:perona-malik-denoising-map}]{\noindent \begin{centering}
\includegraphics[width=0.4\textwidth]{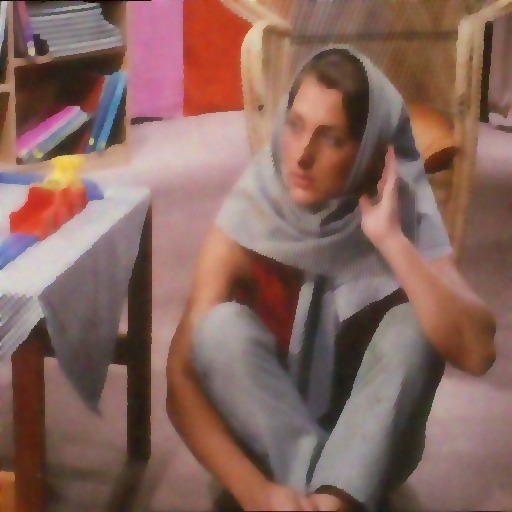}
\par\end{centering}

} & \subfloat[Mean edge weights corresponding to MAP-assignment\label{fig:perona-malik-denoising-map-edge}]{\noindent \begin{centering}
\includegraphics[width=0.4\textwidth]{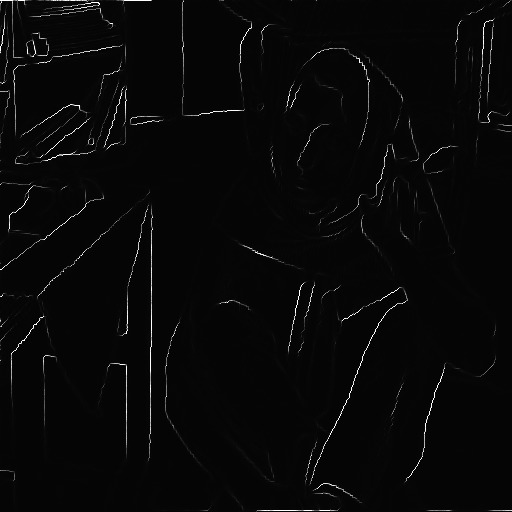}
\par\end{centering}

}\tabularnewline
\end{tabular}
\par\end{centering}

\caption[Denoising results for Perona-Malik prior]{Denoising results for Perona-Malik prior with $\lambda=C=10^{3}$
and Gaussian noise with $\sigma=0.1$.\label{fig:perona-malik-denoising}}
\label{fig:perona-malik-1-denoising}
\end{figure*}

Figure \ref{fig:perona-malik-denoising-map} shows the MAP-assignment
that we obtained by applying the EM algorithm to the image in Figure \ref{fig:perona-malik-denoising-corrupted} and Figure \ref{fig:perona-malik-denoising-meanfield}
shows the result from the (relaxed) mean field algorithm. We see that
the mean field algorithm finds more edges and better restores the
finer details in the image. This can also be seen in Figure \ref{fig:perona-malik-denoising-map-edge}
and Figure \ref{fig:perona-malik-denoising-meanfield-edge}, where
the corresponding mean edge weights $1/\xi_{0}$ are shown. Whereas
the EM algorithm tends to make a binary decision whether a given pixel
is part of an edge or not, the mean field algorithm also finds some
soft edges and textured areas in the image. Similarly, Figure \ref{fig:perona-malik-deconvolution}
shows the results obtained by applying the same Perona-Malik-prior as for the denoising problem to a
deconvolution problem. Again, we see that the mean field approximation
in Figure \ref{fig:perona-malik-deconvolution-meanfield} better captures
some of the finer details in the uncorrupted image than the corresponding
MAP-assignment in Figure \ref{fig:perona-malik-deconvolution-map}.

\begin{figure*}[ht]
\noindent \begin{centering}
\begin{tabular}{cc}
\subfloat[Uncorrupted image]{\noindent \begin{centering}
\includegraphics[bb=10bp 0bp 768bp 500bp,clip,width=0.45\textwidth]{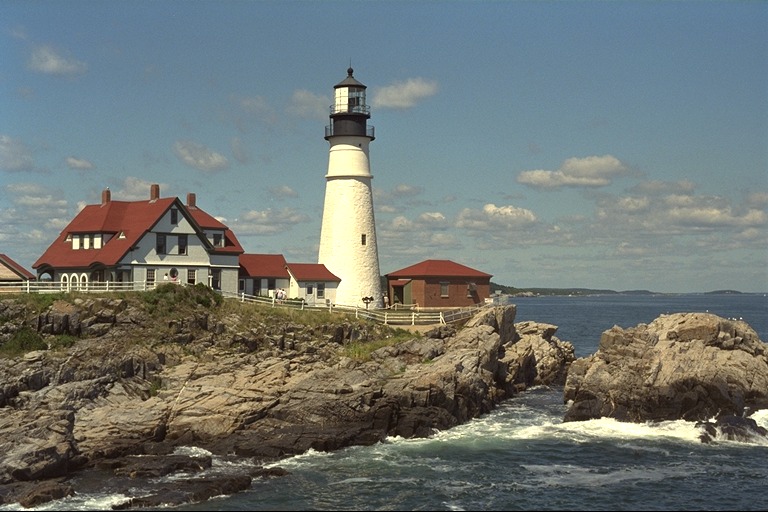}
\par\end{centering}

} & \subfloat[Blurry image]{\noindent \begin{centering}
\includegraphics[bb=10bp 0bp 768bp 500bp,clip,width=0.45\textwidth]{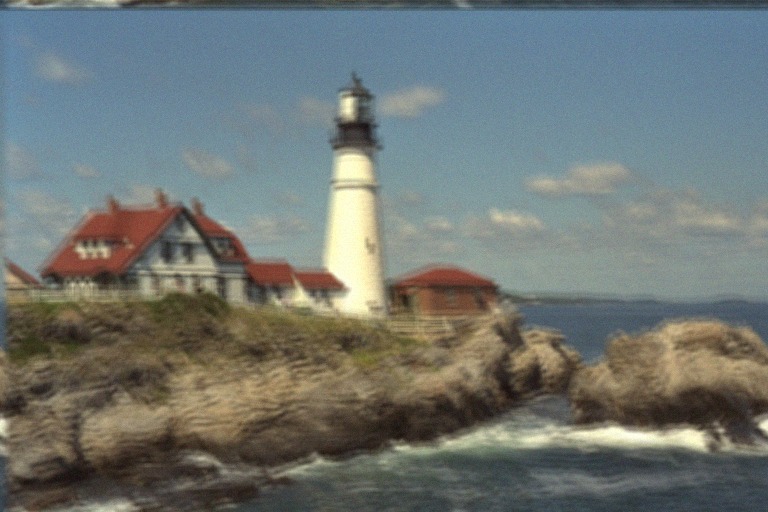}
\par\end{centering}

}\tabularnewline
\subfloat[Mean field approximation\label{fig:perona-malik-deconvolution-meanfield}]{\noindent \begin{centering}
\includegraphics[bb=10bp 0bp 768bp 500bp,clip,width=0.45\textwidth]{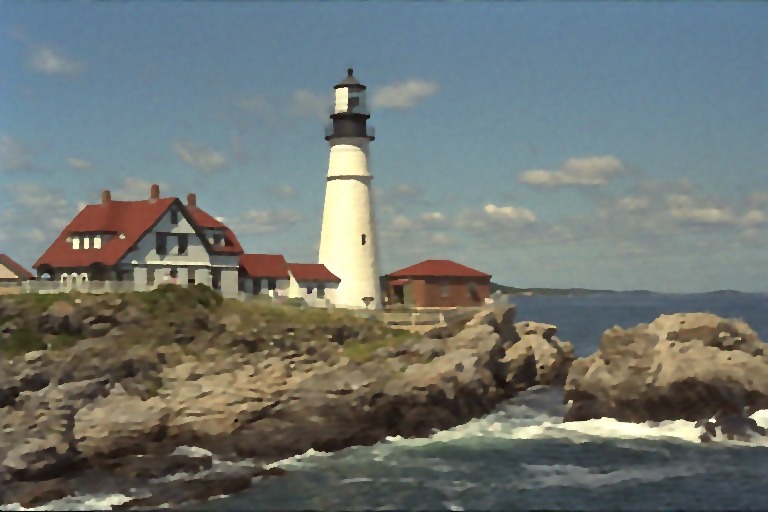}
\par\end{centering}

} & \subfloat[MAP-assignment\label{fig:perona-malik-deconvolution-map}]{\noindent \begin{centering}
\includegraphics[bb=10bp 0bp 768bp 500bp,clip,width=0.45\textwidth]{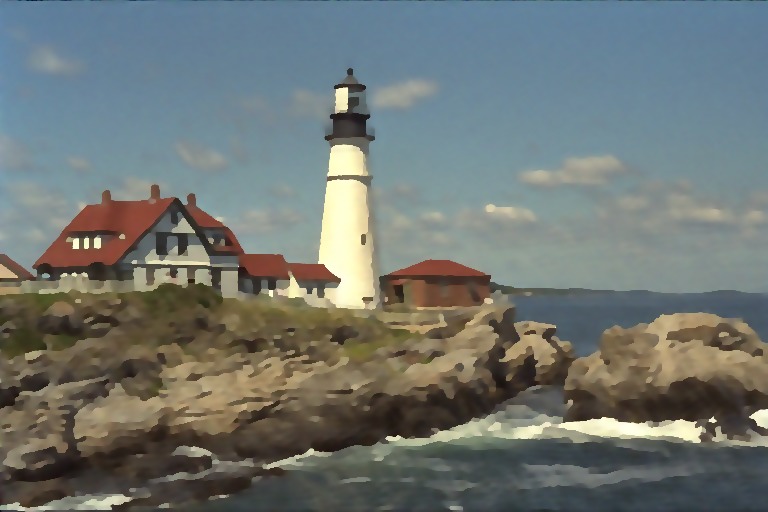}
\par\end{centering}

}\tabularnewline
\end{tabular}
\par\end{centering}

\caption[Deconvolution results for Perona-Malik prior]{Deconvolution results for Perona-Malik prior with $\lambda=C=4\cdot10^{3}$
and Gaussian noise with $\sigma=0.02$.\label{fig:perona-malik-deconvolution}}
\end{figure*}

Our last example shows that also discrete measures $q$ in~\eqref{eq:scale-mixture-model} can be used.
A discrete prior for the latent edge weight leads to a binary decision if a pixel is considered to be an edge pixel or not.
We define the prior by the counting measure $q$ concentrated on $\{0,\lambda\}$, i.e. $q = \delta_0 + \delta_\lambda$ and use the function
$v(z) = -\tfrac{\mu}{\lambda} z$.
This yields
\begin{align*}
  \psi(t) & = -\log\Big(\int e^{-\big(t-\tfrac{\mu}{\lambda})z}q(dz)\Big)\\
  & = -\log\Big(\int e^{-\big(t-\tfrac{\mu}{\lambda})z}(\delta_{0}+\delta_{\lambda})(dz)\Big)\\
  & = -\log\big(1 + e^{-\lambda t+\mu}\big)\\
  & = \log \sigma (\lambda t - \mu),
\end{align*}
where $\sigma$ is the sigmoid function given by $\sigma(t)=\tfrac{1}{1 + \mathrm e^{-t}}$. Because $\sigma^\prime(t) = \sigma(t)(1-\sigma(t))$, this shows that
\[
 f(t) = \psi^\prime (t) = \lambda\left( 1 - \sigma(\lambda t - \mu) \right).
\]
Intuitively, $z(x) = 0$ indicates that a given pixel $x$ belongs to an edge, while $z(x)=\lambda$ indicates the opposite. Note that this prior can therefore be interpreted as a probabilistic version
of the Mumford-Shah functional \cite{mumford1989optimal}.
While connections between the Perona-Malik model and the Mumford-Shah model have been observed previously in~\cite{morini2003mumford} where the Mumford-Shah functional appeared as the $\Gamma$-limit of Perona-Malik models for dicretized $\Omega$ while the discretization gets finer and finer, we obtain both models in the same discretized context.
Figure \ref{fig:edge-denoising-cm}
shows the conditional mean computed from the Markov chain after $100$
iterations of Gibbs sampling. Figure \ref{fig:edge-denoising-map}
shows the corresponding MAP-assignment computed using the EM-algorithm.
We see that in this case the conditional mean is much better at restoring
edges and fine details than the corresponding MAP-assignment. Figure
\ref{fig:edge-denoising-cm-edge} and Figure \ref{fig:edge-denoising-map-edge},
which show the corresponding mean edge images, confirm this hypothesis.
Better MAP-reconstructions can be obtained by changing the $\mu$-parameter,
but this corresponds to a different prior distribution.

\begin{figure*}[ht]
\noindent \begin{centering}
\begin{tabular}{cc}
\subfloat[Uncorrupted image]{\noindent \begin{centering}
\includegraphics[width=0.4\textwidth]{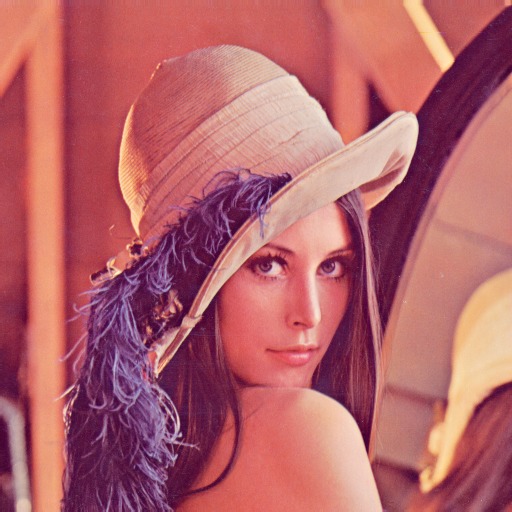}
\par\end{centering}

} & \subfloat[Noisy image]{\noindent \begin{centering}
\includegraphics[width=0.4\textwidth]{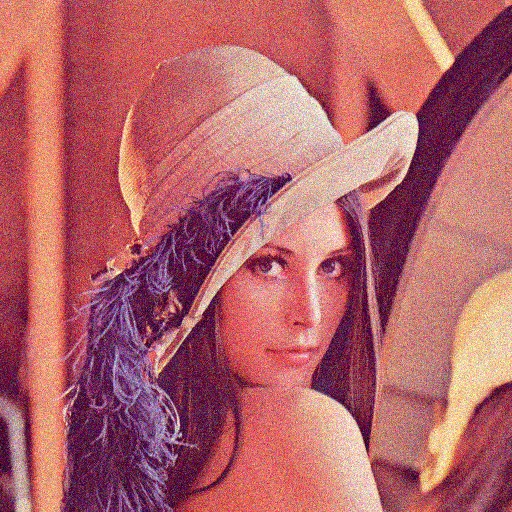}
\par\end{centering}

}\tabularnewline
\subfloat[Conditional mean\label{fig:edge-denoising-cm}]{\noindent \begin{centering}
\includegraphics[width=0.4\textwidth]{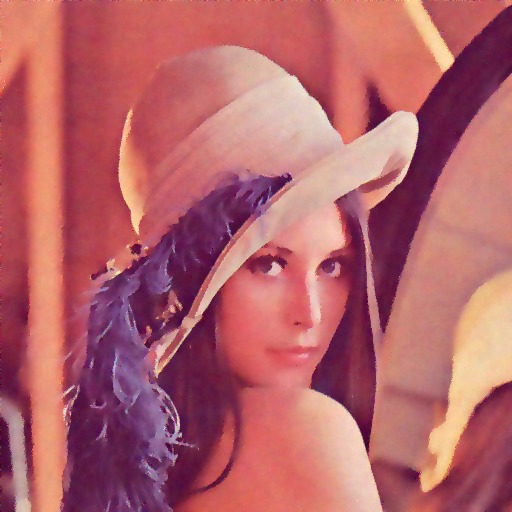}
\par\end{centering}

} & \subfloat[Conditional mean of corresponding edge image\label{fig:edge-denoising-cm-edge}]{\noindent \begin{centering}
\includegraphics[width=0.4\textwidth]{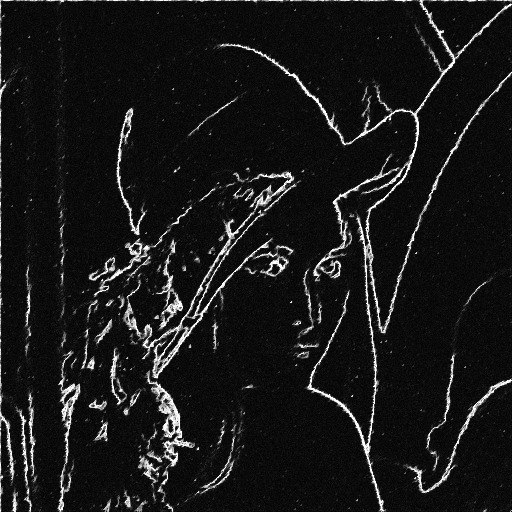}
\par\end{centering}

}\tabularnewline
\subfloat[MAP-assignment\label{fig:edge-denoising-map}]{\noindent \begin{centering}
\includegraphics[width=0.4\textwidth]{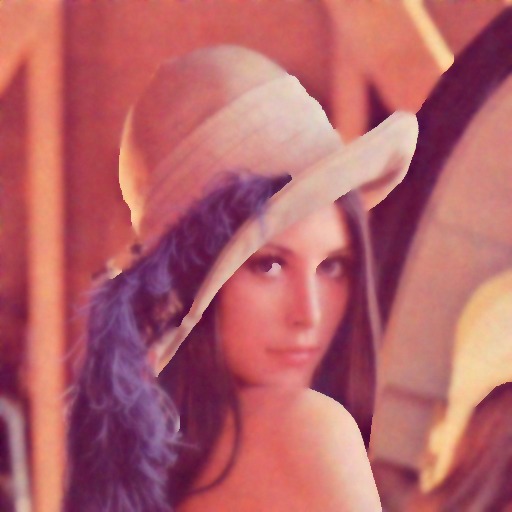}
\par\end{centering}

} & \subfloat[Mean edge weights corresponding to MAP-assignment\label{fig:edge-denoising-map-edge}]{\noindent \begin{centering}
\includegraphics[width=0.4\textwidth]{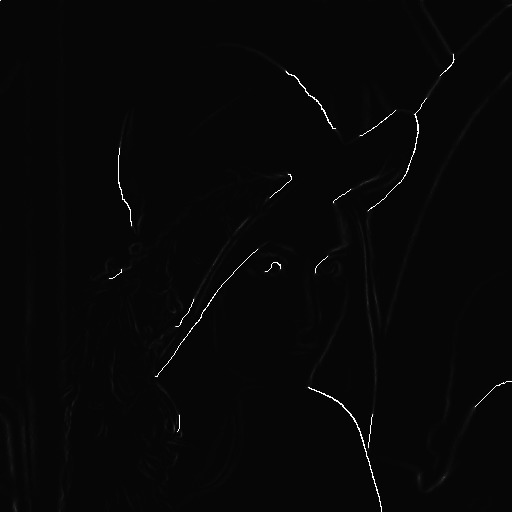}
\par\end{centering}

}\tabularnewline
\end{tabular}
\par\end{centering}

\caption[Denoising results for Mumford-Shah-prior]{Denoising results for the discrete Mumford-Shah-like edge prior with $\lambda=800.0$ and $\mu=3.8$
and Gaussian noise with $\sigma=0.1$.}
\label{fig:edge-denoising}
\end{figure*}


\section{Conclusion}
\label{sec:conclusion}
In this work we established the relationship between the celebrated Perona-Malik model with a probabilistic model for image processing.
We used Gaussian scale mixtures where we modeled the (inverse) variance of a Gaussian smoothness prior as a latent variable. 
We proposed different algorithmic approaches to infer information (usually images and edge maps) from the corresponding posterior and all algorithms resemble the lagged-diffusivity scheme for the Perona-Malik model in one way or another.
We suspect that the interpretation of the Perona-Malik model as a probabilistic model with a latent variable for the edge prior can be related to the underlying neurological motivation for non-linear diffusion models in human image perception as, e.g. outlined in early works of Grossberg at al., see e.g. \cite{cohen1984neural,grossberg1984outline,grossberg1985neural}.

Our interpretation of the Perona-Malik model as an EM algorithm explains the observed over-smoothing and staircasing in the sense that lagged-diffusivity approximates a MAP estimator of the posterior, which is in general not a good representative of the distribution.
Our method based on mean field approximation from Section~\ref{sec:meanfield} partly avoids this over-smoothing and staircasing effect
by explicitly incorporating the uncertainty in the image variable $u$.
However, the mean field approach in its plain form leads to a method with high computational cost and we proposed an approximate mean field method in Algorithm~\ref{alg:gsm-meanfield-approx}.
The approximation is based on a diagonal approximation $C$ of a covariance matrix. While this already leads to good results,  a possible improvement may be to restrict $C$ to the set of
$k\times k$ block matrices. Alternatively we could also restrict
it to the set of circular matrices. Both approximations can also be
combined by setting $C$ to a product of the form
\begin{equation*}
C=C_{block}C_{circ}C_{block}^{\intercal}
\end{equation*}
or even
\begin{equation*}
C=\left(\prod_{i}C_{block}^{(i)}C_{circ}^{(i)}\right)\left(\prod_{i}C_{block}^{(i)}C_{circ}^{(i)}\right)^{\intercal}
\end{equation*}
yielding better and better approximations to the true covariance matrix.


\section*{Acknowledgements}
 We would like to thank Sebastian Nowozin from Microsoft Research for
 some helpful literature hints.

\bibliographystyle{plain}
\bibliography{probmod}

\end{document}